\documentclass[journal]{IEEEtran}
\usepackage{multirow}
\usepackage{hhline}

\ifCLASSINFOpdf
  \usepackage[pdftex]{graphicx}
  \graphicspath{{../pdf/}{../jpeg/}}
  \DeclareGraphicsExtensions{.pdf,.jpg,.png}
\else
  \usepackage[dvips]{graphicx}
  \graphicspath{{../eps/}}
  \DeclareGraphicsExtensions{.eps}
\fi
\usepackage{booktabs}

\usepackage{amsthm, amsmath,amssymb}
\newtheorem{lemma}{Lemma}
\newtheorem{theorem}{Theorem}
\theoremstyle{definition}
\newtheorem{definition}{Definition}[section]
\DeclareMathOperator*{\argmax}{arg\,max}

\usepackage{algorithm}
\usepackage{algorithmic}

\makeatletter
\newcommand\fs@norules{\def\@fs@cfont{\bfseries}\let\@fs@capt\floatc@ruled
  \def\@fs@pre{}%
  \def\@fs@post{}%
  \def\@fs@mid{\kern3pt}%
  \let\@fs@iftopcapt\iftrue}
\makeatother
\restylefloat{algorithm}

\ifCLASSOPTIONcompsoc
  \usepackage[caption=false,font=normalsize,labelfont=sf,textfont=sf]{subfig}
\else
  \usepackage[caption=false,font=footnotesize]{subfig}
\fi
\begin{document}
%
\title{Wearable Affective Life-Log System for Understanding Emotion Dynamics in Daily Life}
%
%
%

\author{Byung~Hyung~Kim
        and~Sungho~Jo,~\IEEEmembership{Member,~IEEE}
\thanks{B. Kim and S. Jo are with the School of Computing, KAIST, Republic of Korea.  e-mail: \{bhyung, shjo\}@kaist.ac.kr.}
\thanks{S. Jo is the corresponding author.}
}

\maketitle

\begin{abstract}
Past research on recognizing human affect has made use of a variety of physiological sensors in many ways. Nonetheless, how affective dynamics are influenced in the context of human daily life has not yet been explored. In this work, we present a wearable affective life-log system (ALIS), that is robust as well as easy to use in daily life to detect emotional changes and determine their cause-and-effect relationship on users' lives. The proposed system records how a user feels in certain situations during long-term activities with physiological sensors. Based on the long-term monitoring, the system analyzes how the contexts of the user's life affect his/her emotion changes. Furthermore, real-world experimental results demonstrate that the proposed wearable life-log system enables us to build causal structures to find effective stress relievers suited to every stressful situation in school life. 
\end{abstract}

%
\IEEEpeerreviewmaketitle

\section{Introduction}
%
%
%
%
\IEEEPARstart{P}{eople} experience various emotions from a single event in different situations. For instance, today's coffee is not always the same as yesterday's coffee. The cup of coffee we drank today may not be as enjoyable as the cup of coffee we drank yesterday. While drinking coffee generally helps to reduce a person's stress, the stress-relieving effects of coffee may vary from day to day for many reasons. For a person who likes calm and quiet surronding, a cup of coffee drunk today in a crowded coffee shop with distracting background noise is likely to be less enjoyable than a cup of coffee drunk yesterday in the quiet kitchen of one's own home. This instance shows that a person can have different emotional responses to the same life events in different circumstances.

Why and how does a person experience various emotions from a single event under different situations? Answering this question could improve human life in a variety of ways, as by improving physical health. People who suffer from depression are more vulnerable to heart disease than people with no history of depression. Therefore, discovering life elements related to depression and offering guidance to avoid such elements can help sufferers to lessen their suffering and lead a meaningful life. In response to this question, recent researches on recognizing human affect has made use of a variety of physiological sensors in many ways. Such miniaturized physiological sensors and advanced mobile computing technologies have enabled people to monitor their physiological signals continuously in daily life with so-called ``everyday technology''~\cite{garbarino2014empatica,picard2016multiple,williams2015swarm,yuruten2014predictors}. Using these sensors, features such as heart rate variability, pulse oximetry, and galvanic skin response have been used to capture emotional changes so as to help us understand the etiology of mental health pathologies such as stress. 

However, how affective dynamics are influenced in the context of human daily life has not yet been explored. Most studies~\cite{zheng2018emotionmeter} have been limited to laboratory environments assessed by self-report tools~\cite{bradley1994measuring} or specialized questionnaires. These methods, however, provide evidence only on immediate affect from a single event and provide only a limited understanding of affective dynamics. In our study, understanding affective dynamics surrounding daily life mainly means discovering the causal relations between life contents and human affect in daily life. This understanding includes the correct affect elicitation mechanism and its effect on the emotional response. However, building reliable automated systems for understanding affective dynamics is a challenging problem involving various noises, a low signal-to-noise ratio (SNR) of physiological signals, inter- and intra-subject variability, and usability.

\subsection{Contribution}
Solving these challenges, we present a wearable affective life-log system, ALIS, which helps to bridge the gap between the low-level physiological sensor representation and the high-level context sensitive interpretation of affect. The proposed system records how a user feels in certain situations during long-term activities using physiological sensors. Based on the long-term monitoring, the system analyzes how the contexts of the user's life affect his/her emotion changes. Notably, the contributions of the proposed system, as an alternative to existing works, are: 
\begin{itemize}
  \item A wearable life-log system for capturing human affect: Easy to use in daily life. The proposed system is designed to provide comfort and performance to users during their long-term activities. Our device offers user advanced usability to users in their daily life. 
  \item A physiological affect model for capturing and tracing affective dynamics in real-world scenarios: By extending our previous work on physiological affect-based emotion recognition~\cite{kim2018deep}, we present a physiological affect network to learn affective dynamics with minimal supervision. Unlike previous works~\cite{zheng2018emotionmeter,valenza2014revealing} that focused on a specific personalized assessment for every event, our model captures affective states and continuously traces their changes, exploiting unlabeled and unbalanced real-world data through semi-supervised learning. 
  \item An asymmetric causal influence model of the life content-affect relationship: We present an asymmetric measure to identify the causal relationship between affective contents and human emotion in daily life. The model for the first time allows users to understand when, what, and how their surroundings affect them unawares in their daily life. 
  \item Extensive experiments to evaluate our proposed system in a long-term series of life-logging over several days in real-world scenarios, through the evaluation of which we demonstrate how emotion dynamics are influenced by the affective contents of human daily life using various contextual information acquired over up to 60 consecutive days.
\end{itemize}

\subsection{Problem Statements and Challenge Issues}
Our goal is to determine affective causality in daily life by analyzing emotional states and life contents encountered in various situations. We present the problem formulation of modeling the causal relation between life contents and emotional changes in the proposed ALIS. Without loss of generality, we treat the problem as a combination of two-stage problems. The first stage is the emotion recognition problem. Given physiological signals, the problem is the identification of the correct emotional state as
\begin{small}
\begin{equation}
  \hat{y}=\text{argmax}_{y \in \mathcal{Y}}P(y|\mathcal{X}_1, \mathcal{X}_2,\dots, \mathcal{X}_t),
\label{eq:ProblemStatementEmotion}
\end{equation}
\end{small}where $\mathcal{X}_t$ is a segment of physiological signals at time $t$ and $\mathcal{Y}$ is the set of emotional states such as happiness, surprise, anger, fear, and sadness.

The second stage is the affective causality learning problem. The problem is formulated as a graph 
\begin{small}
\begin{equation}
  \mathcal{G}=(\mathcal{M}, \mathcal{C}, E),
\label{eq:ProblemStatementCausality}
\end{equation}
\end{small}where $\mathcal{M} \in \mathbb{R}^{N_s \times N_e \times T}$ is the emotion occurrence matrix, $N_s$ is the number of situations and $N_e$ is the number of emotional states. $\mathcal{C} \in \mathbb{R}^{N_s \times N_c \times T}$ is the life contents occurrence matrix, where $N_c$ is the number of contents and $e_{ij} \in E$ indicates the affective causal effect of sequence $\mathcal{C}_i$ on sequence $\mathcal{M}_j$. 

The proposed ALIS is built upon our previous work on recognizing human emotion~\cite{kim2018deep}. Although the system has shown merits in discriminating physiological signals in different emotional states, non-trivial and challenging issues exist when the model is applied to demonstrate its efficacy in daily life. There are two main challenging problems in real-world environments as follows.
\subsubsection{Limited and unbalanced labeling problem in emotion recognition (C.1)}Despite the reliablility of the prior work in emotion recognition, obtaining accurate results of emotion recognition becomes a severe challenge when physiological signals accompany unreliable labels. In real-world scenarios, only a few labels are available where humans rate their feeling in indoor environments using self-reporting tools. The limited and unbalanced class labels deteriorate the performance of classifiers in a broad range of emotion classification problems in (\ref{eq:ProblemStatementEmotion}). 
\subsubsection{Sparsity problem in affective causality identification (C.2)}Discovering affective causality is essential to the understanding of users' affective dynamics in their daily life. However, it is unclear whether emotional causality exists in real-world situations where a user encounters life events in long temporal sequences. For the most part the issue of how to identify the causality direction and discover the influence of latent factors are challenging, as emotions are usually affected by various complex and subtle factors in daily life.

ALIS is designed to solve the two problems (\ref{eq:ProblemStatementEmotion}) and (\ref{eq:ProblemStatementCausality}), addressing the challenge issues (C.1) and (C.2) along with advanced wearable sensor technologies for multi-modal acquisition toward a deeper understanding of how humans act and feel while interacting with their real-world environment.  

The rest of this paper is organized as follows: We describe prior work in Section II with respect to emotion recognition, causal inference, and life-logging technologies for monitoring long-term activities without loss of usability. Section III covers the preliminaries of our previous work and causal inference. In Section IV, we present our system, describing the system design and framework. Sections V and VI evaluate our system on synthetic, public, and real data sets. Lastly, we concludes this paper with a discussion of future work in Section VII.

\section{Related Works}
\subsection{Physiology and Multi-Modality for Recognizing Human Affect}
Several physiology-based studies of human affect have been conducted that have advanced significantly in many ways over the past few decades~\cite{cornelius1996science,sander2005systems}. Along with this research area, we focus on using machine learning to identify physiological patterns that corresponds to the expressions of different emotions. In particular, EEG measures the electrical activity of the brain by recording its spontaneous electrical activity with multiple electrodes placed on the scalp. Despite its low spatial resolution on the scalp, its very high temporal resolution, non-invasiveness, and mobility are valuable in real-world environments such as clinical applications~\cite{smith2005eeg, lovato2014meta}. 

Most EEG-based emotion recognition systems have extracted and selected EEG-based features through electrode selection based on neuro-scienctific assumptions. Zheng~\textit{et al.} have collected EEG signals above the ears and presented a framework to classify four emotions~\cite{zheng2018emotionmeter}. They extracted EEG features from the six symmetrical temporal electrodes and showed that EEG has the advantage of classifying happy emotions. Similarly, Wang~\textit{et al.} investigated the characteristics of EEG features to assess the relationships with emotional states~\cite{wang2014emotional}. Their study showed that the right occipital and the parietal lobes are mainly associated with emotions in the alpha band, while the temporal lobes are associated with emotions in the beta band. Furthermore, the left frontal and right temporal lobes are associated with emotions in the gamma band. Petrantonakis and Leontios developed adaptive methods to investigate physiological associations between EEG signal segments and emotions in the time-frequency domain~\cite{petrantonakis2012adaptive}. They exploited the frontal EEG asymmetry and the multidimensional directed information approach to explain causality between the right and left hemispheres. These results have shown that emotional lateralization in the frontal and temporal lobes can be a good differentiator of emotional states.

EEG-based emotion recognition systems have often shown improved results when different modalities have been used~\cite{koelstra2012deap,soleymani2016analysis,subramanian2016ascertain}. Among the many peripheral physiological signals, PPG, which measures blood volume, is widely used to compute heart rate (HR). It uses optical-based technology to detect volumetric changes in blood in peripheral circulation. Although its accuracy is considered lower than that of electrocardiograms (ECGs), it has been used to develop wearable biosensors in clinical applications such as detecting mental stress in daily life~\cite{zhang2015troika} due to its simplicity. HR, as well as heart rate variability (HRV), has been shown to be useful for emotion assessment~\cite{chigira2011area,lyu2015measuring,sun2014moustress}. Over the past two decades, some reports have shown that HRV analysis can provide a distinct assessment of autonomic function in both the time and frequency domains. However, these assessments require high time and frequency resolutions. Due to these requirements, HRV has only been suitable for analyzing long-term data. Several researchers have focused on overcoming this limitation. Valenza~\textit{et al.}~\cite{valenza2014revealing} have recently developed a personal probabilistic framework to characterize emotional states by analyzing heartbeat dynamics exclusively to assess real-time emotional responses accurately. 

In these studies, distinct or peaked changes of physiological signals in the time or frequency domains at a single instantaneous time have been considered candidates. However, this approach is limited and cannot be used to fully describe emotion elicitation mechanisms due to their complex nature and multidimensional character. To overcome this problem, this work formulates emotion recognition as a spectral-temporal physiological sequence learning problem.
\subsection{Causality}
Understanding causal direction is essential to predicting the consequence of any intervention from a group of observation samples and is critical to many applications, including biology and social science~\cite{cai2017understanding}. We should note that causal learning is different from mainstream statistical learning methods in that it aims to discover the data-generation mechanism instead of characterizing the joint distribution of the observed variables; that is the most significant difference between causality and correlation. Discovering emotional influence has been a focus for testing whether the correlation exists in real-world applications. Wu~\textit{et al.}~\cite{wu2017inferring} developed a model to learn the emotional tags of social images and to verify the correlation between user demographics and emotional tags of social images. Jia~\textit{et al.}~\cite{jia2012can} exploited the social correlation among images to infer emotions from images. Their results showed that the emotions evoked by two consecutive images can be associated with each other from the same user uploads. Yang~\textit{et al.}~\cite{yang2016social} reported that the ability to emotionally influence others is closely associated with a user's social roles in image social networks. 

\subsection{Life-Log Technologies for Monitoring Long-Term Activities}
Self-tracking through logging various personal metrics is itself is not new. For instance, people with chronic conditions such as diabetes have long been monitoring their personal metrics to figure out how daily habits influence their symptoms. From a physiological point of view, EDA has been used to measure emotional changes for over 125 years~\cite{picard2016multiple}.  However, these methods have all been limited to highly controlled environments because of the complex and cumbersome machinery involved. In order to monitor long-term activities in any environments, wearable technologies are required to be ``as natural and unnoticeable as possible". Over the past decade, there have been significant advances in miniaturizing EDA sensors. Pho et al.~\cite{poh2010wearable} developed a compact and low-cost wearable EDA sensor to monitor long-term activities during daily activities outside of a laboratory. Garbarino et al.~\cite{garbarino2014empatica} presented a wearable wireless multi-sensor device, ``Empatica E3," which has four embedded sensors: PPG, EDA, 3-axis accelerometer, and temperature. The device is a new wrist-wearable sensor that is suitable for real-life applications. Wearable design techniques also have been developed satisfying the conditions for wearability. Williams et al.~\cite{williams2015swarm} presented SWARM, a wearable affective technology which includes a modular soft circuit design technique. SWARM is simply a scarf, but its modular actuation components can accommodate users' sensory capabilities and preferences so that the users can cope with their emotions. 

\section{Preliminaries}
\subsection{Deep Physiological Affect Network (DPAN) for Recognizing Human Emotions}
DPAN describes affect elicitation mechanisms used to detect emotional changes reflected by physiological signals. It takes two-channeled EEG signals underlying brain lateralization and a PPG signal as inputs and outputs a one-dimensional vector representing emotional states scaled from 1 to 9. Suppose that DPAN obtains the physiological signals at time $N$ over a spectral-temporal region represented by an $M \times N$ matrix with $P$ different modalities. From the two modalities of EEG and PPG sensors, physiological features $B_t$ and $H_t$ are extracted from the respective sensors as follows. 
\begin{small}
\begin{equation}
\label{eq:BrainLateralizationFeature}
    B_t = \xi_{rl}\circ\frac{(\zeta_l-\zeta_r)}{(\zeta_l+\zeta_r)},
\end{equation}
\end{small}where `$\circ$' denotes the Hadamard product. $\frac{(\zeta_l-\zeta_r)}{(\zeta_l+\zeta_r)}$ represents the spectral asymmetry and the matrix $\xi_{rl}$ is the causal asymmetry between the $r$ and $l$ EEG bipolar channels. The brain asymmetry feature $B_t$ describes the directionality and magnitude of emotional lateralization between the two hemispheres.

DPAN extracts the heart rate features $H_t$ over the $M \times N$ spectral-temporal domain, where frequencies with peaks in the PSD of the PPG signal are regarded as candidates of the true heart rate from the PPG signal $P_t$ at each time frame $t$. These data form a candidate set over time. The observation at a given time can then be represented by a tensor $\mathcal{X} \in \mathbb{R}^{M \times N \times P}$, where $\mathbb{R}$ denotes the domain of the observed physiological features. Then the learning problem then is the identification of the correct class based on the sequence of tensors $\mathcal{X}_1, \mathcal{X}_2, \dots, \mathcal{X}_t$. 
\begin{small}
\begin{equation}
\hat{y}=\argmax_{y \in \mathcal{Y}}P(y|\mathcal{X}_1, \mathcal{X}_2,\dots, \mathcal{X}_t), 
\end{equation}
\end{small}where $\mathcal{Y}$ is the set of valence-arousal classes. 
\begin{figure}[t]
    \centering
    \includegraphics[width=1\columnwidth]{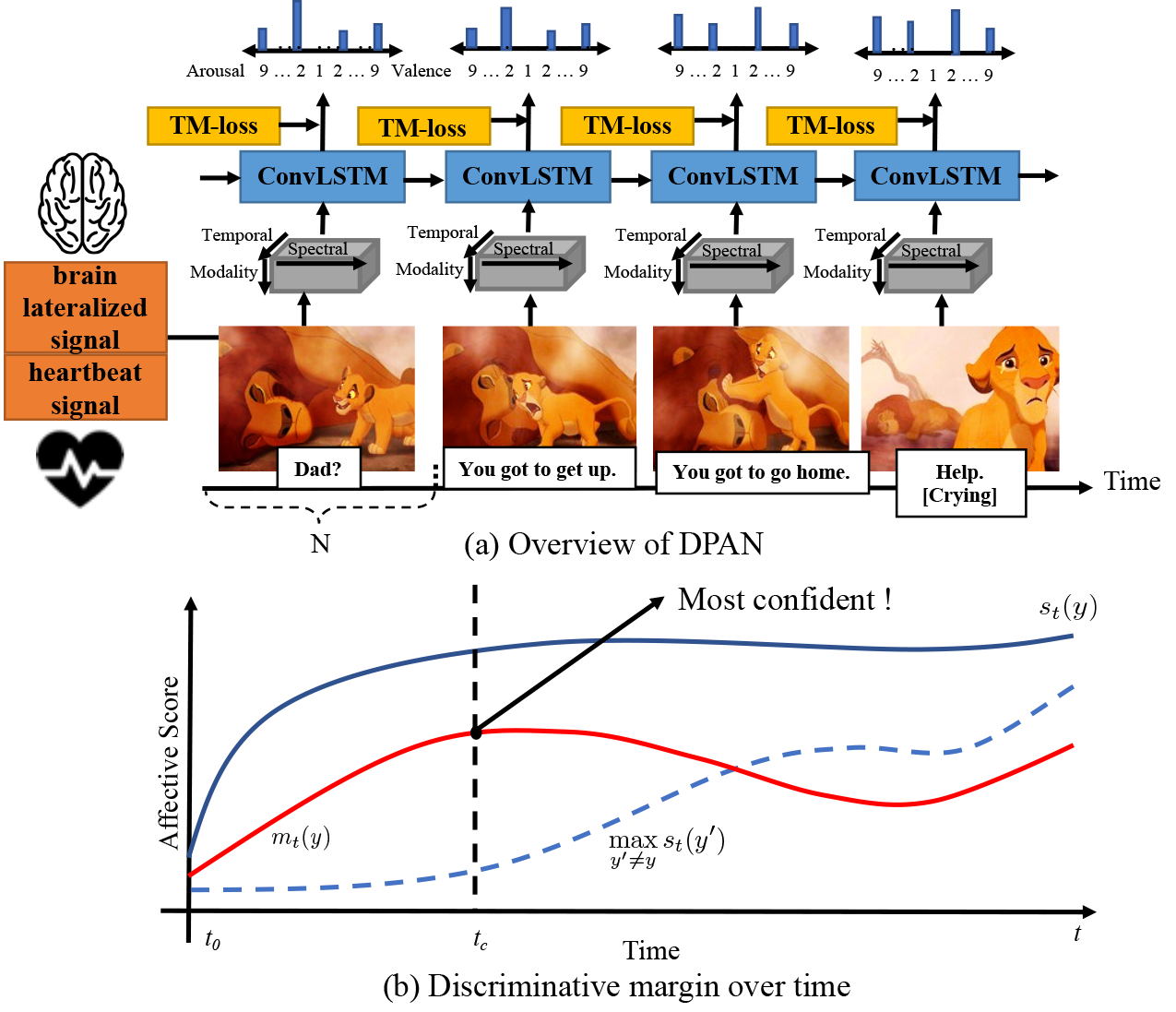}
    \caption{(a) An overview of DPAN. After every time interval $N$, the proposed DPAN first extracts two physiological features (brain lateralized and heartbeat features) and constructs a spectral-temporal tensor. These features are then fed into ConvLSTM to compute affective scores of emotions via our proposed loss model, temporal margin-based loss (TM-loss). The output at the final sequence is selected to represent an emotion over a 2-dimensional valence-arousal model for the entire sequence. (b) The discriminative margin $m_t(y)$ (red line) of an emotion $y$ started at $t_0$. The margin $m_t(y)$ is computed as the difference between the ground truth affective score $s_t(y)$ (blue line) and the maximum scores $\max_{y' \neq y}s_t(y')$ (dashed blue line) of all incorrect emotion states between $t_0$ and $t$. The model becomes more and more confident in classifying emotion states until the time $t_c$. However, after the time $t_c$, $\mathcal{L}_t$ are non-zero due to the violation of the monotonicity of the margin.}
    \label{fig:OverviewDPAN}
\end{figure}
\figurename~\ref{fig:OverviewDPAN} shows the entire overview of DPAN for the recognition of emotions. To solve the learning problem, DPAN feeds spectral-temporal tensor-based physiological features into ConvLSTMs to compute affective scores of emotions via the proposed loss model, temporal margin-based loss(TM-loss). The proposed TM-loss aims to learn the progression patterns of the emotions in training for developing reliable affect models. The TM-loss is a new formulation based on the temporal margin between the correct and incorrect emotional states. The reasoning for using the formulation is as follows: 
\begin{itemize}
    \item \textit{When more of a particular emotion is observed, the model should be more confident of the emotional elicitation as the recognition process progresses.}
\end{itemize}
The function constrains the affective score of the correct emotional state to discriminate its margin, which does not monotonically decrease with all the others while the emotion progresses. 
\begin{small}
    \begin{equation}
	\label{eq:TMloss1}
	\mathcal{L}_t = -\log s_t(y) + \lambda\max(0, \max_{t' \in [t_0, t-1]}m_{t'}(y) - m_t(y)), 
    \end{equation}
\end{small}where $-\log s_t(y)$ is the conventional cross-entropy loss function commonly to train deep-learning models, $y$ is the ground truth of emotion rating, $s_t(y)$ is the classified affective score of the ground truth label $y$ for the time $t$, and $m_t(y)$ is the discriminative margin of the emotion label $y$ at time $t$.
\begin{small}
    \begin{equation}
	\label{eq:TMloss2}
	m_t(y) = s_t(y)- \max\{s_t(y') | \forall y' \in \mathcal{Y}, y' \neq y\}. 
    \end{equation}
\end{small} 
$\lambda \in \mathbb{Z}^+$ is a relative term to control the effects of the discriminative margin. As described in (\ref{eq:TMloss1}) and (\ref{eq:TMloss2}), a model becomes more confident in discriminating between the correct state and the incorrect states over time.
With this function DPAN is encouraged to maintain monotonicity in the affective score as the emotion training progresses. As shown in \figurename~\ref{fig:OverviewDPAN}(b), after the time $t_c$, the loss becomes non-zero due to the violation of the monotonicity of the margin. Note that the margin $m_t(y)$ of the emotion $y$ spanning $[t_0, t]$ is computed as the difference between the affective score $s_t(y)$ for the ground truth $y$ and the maximum classification scores $\max_{y' \neq y} s(y')$ for all incorrect ratings at each time point in $[t_0, t]$. 
\subsection{Conditional Independence Test}
To discover the causality between affective dynamics and life contents and answer (C.2), we test the conditional independence with three sequences $A_i$, $A_j$, and $A_k$. Suppose the three sequences have the same length $T$, the test is to verify the statistical significance of the statement $A_i \perp A_j | A_k$. Considering each triplet ($A_i(t)$, $A_j(t)$, $A_k(t)$) for each $1 \leq t \leq T$ as a sample over three variables, a 3-D contingency table $C$ records the number of triplet samples on the three variables such that 
\begin{small}
    \begin{equation}
	C_{opq} = |\{1 \leq t \leq T | A_i(t) = o, A_j(t) = p, A_k(t) = q \}|.
    \end{equation}
\end{small}Note that the expectation of $C_{opq}$ under the null hypothesis can be estimated by:
\begin{small}
    \begin{equation}
	E(C_{ops})=(C_{*pq}C_{o*q})/(C_{**q}),
    \end{equation}
\end{small}where $C_{*pq}$, $C_{o*q}$, and $C_{op*}$ are the marginals of the counts with $A_i$, $A_j$, $A_k$, respectively. For the test, we use the standard $G^2$ conditional independence test, which returns the Kullback-Leibler divergence between the distributions of $C_{opq}$ and $E(C_{opq})$ over all three variables:
\begin{small}
    \begin{equation}
	G^2 = 2 \sum_{o,p,q} C_{opq} \text{ln} \frac{C_{opq}}{E(C_{opq})},    
    \end{equation}
\end{small}which follows a $\chi^2$ distribution with degree of freedom $(|A_i| -1) * (|A_j| -1) *|A_k|$. For removing the sparsity in the table $C$, the degree of freedom is penalized by the number of zero cells as in \cite{imbens2015causal}. 

\section{Wearable Affective Lifelog System}
\begin{figure}[t]
    \centering
    \includegraphics[width=1\columnwidth]{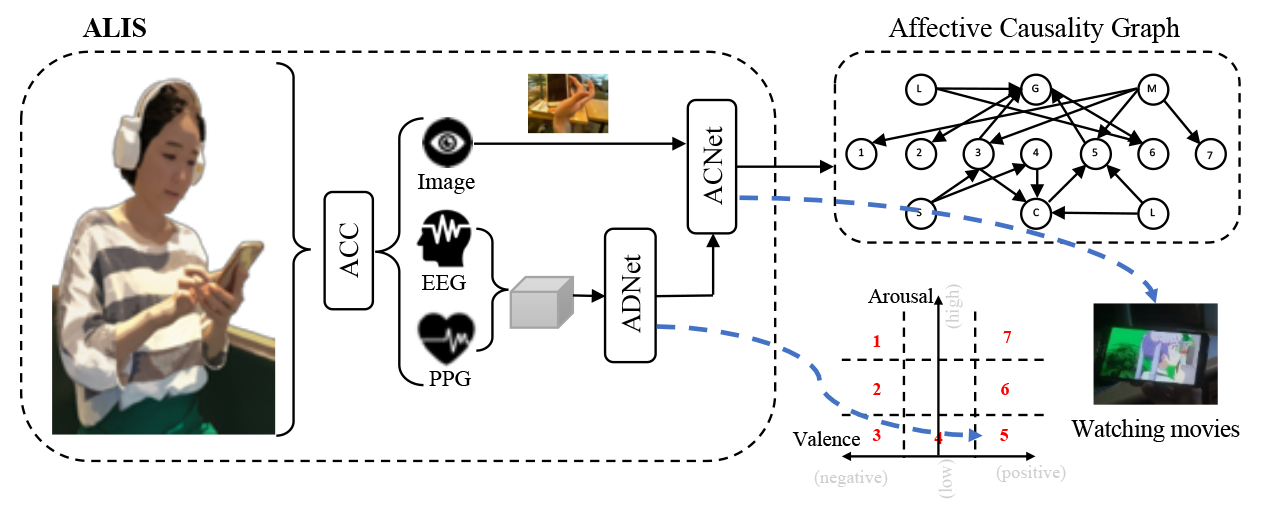}
    \caption{An Overview of ALIS. ALIS consists of an Affective Contents Collector (ACC), Affective Dynamics Network (ADNet), and Affective Causality Network (ACNet). ACC collects an user's contextual information in situations with frontal images and emotion measured by EEG and PPG signals. Given this information, ADNet detects emotional changes, which are used as an input with frontal images for ACNet to discover the causal relationship between emotions and situations.}
    \label{fig:OverviewALIS}
\end{figure}
The Wearable Affective Lifelog System (ALIS) consists of three main parts: The Affective Contents Collector (ACC), Affective Dynamics Network (ADNet), and Affective Causality Network (ACNet). \figurename~\ref{fig:OverviewALIS} shows the entire framework of ALIS. First, ACC gathers contextual information continuously surrounding the wearer in daily life. The logged data are then transferred into ADNet, which aims to find answers about under which situations and to what extent a human feels the elicited affect. ACNet discovers the dynamic causal relationship between the situation faced and the human emotion explored by ADNet. It provides an intuitive understanding of affect dynamics for users. The following sections describe the details of each component.

\subsection{Affective Contents Collector for Logging Data}
\begin{figure}[t]
    \centering
    \includegraphics[width=1\columnwidth]{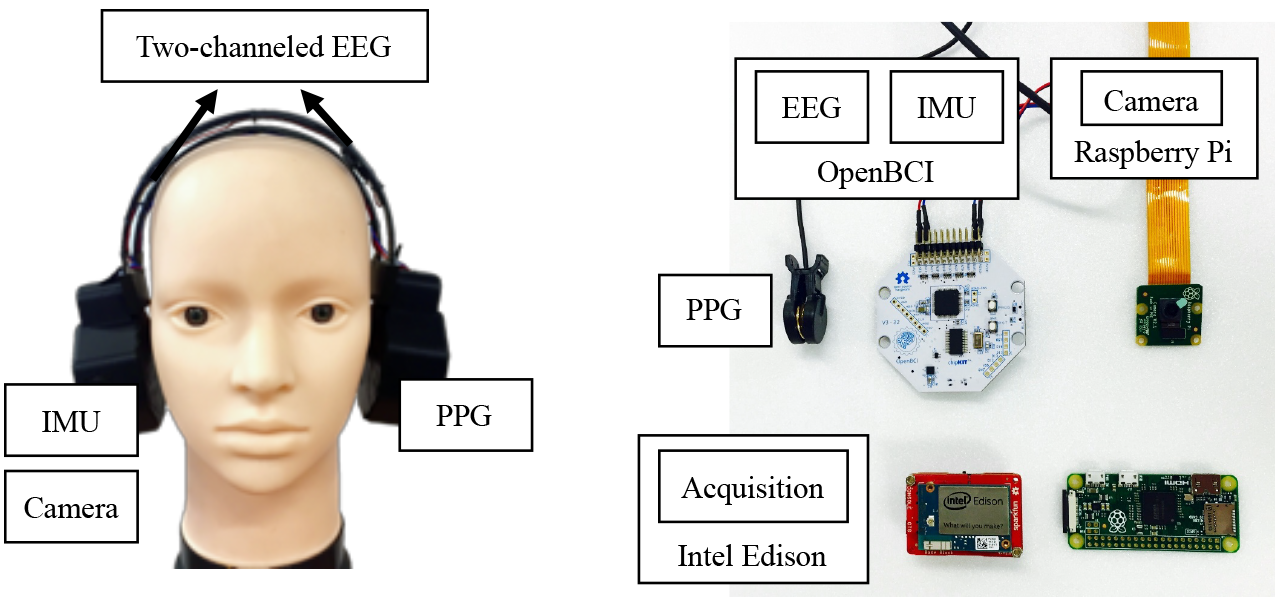}
    \caption{The Affective Contents Collector and its components; IMU, EEG, PPG sensors, and a tiny frontal camera. The location of two EEG electrodes (F3, F4) on the 10-20 international system.}
    \label{fig:OverviewACC}
\end{figure}
ACC is a simple device designed to be easily wearable for users to act freely in everyday situations (See \figurename~\ref{fig:OverviewACC}) as well as to collect human affect correctly. Since human affect is sophisticated and subtle, it is vulnerable to personal, social, and contextual attributes. The noticeability and visibility of wearable devices could elicit unnecessary and irrelevant emotions. Therefore, recording human affect should be unobtrusive when measured in the natural environment. To design an unnoticeable device, we imitated the design of existing easy-to-use wireless headsets. We note that the term ``unobtrusive device'' in this paper means that it is not easily noticed or does not draw attention to itself. The term does not imply that our device aims to be small or concealable. This easy-to-use device provides comfort and performance to users during long-term activities. 

Everyday technology requires wearable systems to have the unprecedented ability to perform the comfortable, long-term, and in situ assessment of physiological activities. However, the development of practical applications is challenging because of the cumbersomeness of the equipment that requires multiple channels to obtain reliable signals and the complexity of setting up experiments. To use body sensors with a guarantee of both reliability and simplicity, we designed the proposed device with a tiny PPG sensor and a minimum configuration of a two-channel EEG, the lowest number of channels necessary to learn patterns of lateralized frontal cortex activity involved in human affect.  

While satisfying the two criteria, our device consists of multimodal sensors to capture various emotions surrounding daily life as follows: 

\begin{itemize}
    \item Frontal Camera for Collecting Visual Contents: Visual information has been widely used to detect situations faced by the user. The study of recognizing scenes and activities by analyzing images from a camera has provided an understanding of contextual information. Hence, in our system, a small frontal viewing camera with a 30 fps sampling rate (Raspberry Pi Zero Camera) was used to record images.  
    \item Small Physiological Sensor to Capture Human Affect: The analysis of patterns of physiological changes has been increasingly studied in the context of affect recognition. To capture this information, we used a two-channel EEG sensor on OpenBCI on the left and right hemispheres and a small ear-PPG sensor, with sampling rates of 250 Hz and 500 Hz, respectively.  
\end{itemize}

\subsection{Affective Dynamics Network for Recognizing Emotions}
\begin{figure}[t]
    \centering
    \includegraphics[width=1\columnwidth]{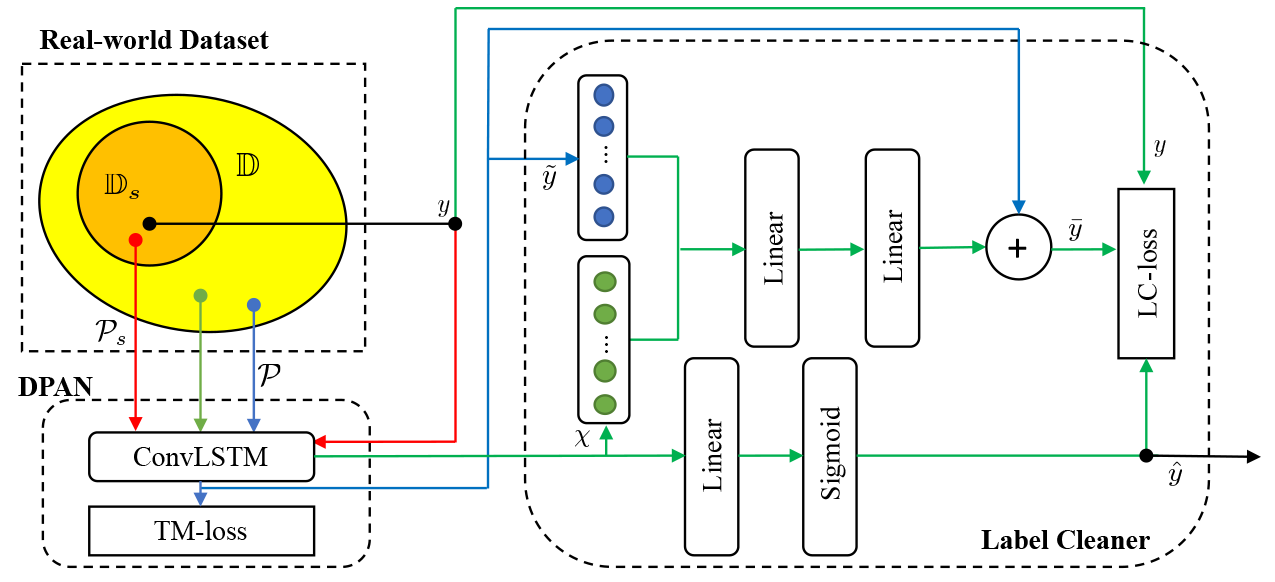}
    \caption{An overview of ADNet. The model first trains physiological signals associated with their labels $y$ on the dataset $\mathbb{D}_s$ by DPAN (red lines). The model then uses the learned parameters on the dataset $\mathbb{D}$ for predicting noisy pseudo-labels $\tilde{y}$ (blue lines), which are fed into a sub-network Label Cleaner conditioned on physiological features $\chi$ from ConvLSTMs. Within the Label Cleaner network, a residual architecture learns the difference between the noisy $\tilde{y}$ and clean labels $y$ (green lines). Finally, the model predicts cleaned labels $\hat{y}$ penalized by a joint loss function (LC-loss).}
    \label{fig:OverviewADNet}
\end{figure}
ADNet aims to solve the problem in (\ref{eq:ProblemStatementEmotion}), addressing the challenging questions (C.1) on a real-world large dataset. Suppose the dataset $\mathbb{D}$ naturally comprises a large number of unlabelled physiological signals $\mathcal{P}$ and a subset $\mathbb{D}_s$ where a small number of ground-truth labels $y$ by self-reporting are available. Our network first learns the representations of the physiological signals $\mathcal{P}$ according to the affective labels on the subset $\mathbb{D}_s$ by DPAN~\cite{kim2018deep} and produces noisy pseudo-labels $\tilde{y}$ on the dataset $\mathbb{D}$. Given the dataset $\mathbb{D} = \{(\mathcal{P}_i, \tilde{y}_i), ... \}$ and its subset $\mathbb{D}_s = \{(\mathcal{P}_i, \tilde{y}_i, y_i), ... \}$, the proposed framework jointly learns to reduce the label noise and predict more accurate labels $\hat{y}$ on the dataset $\mathbb{D}$.

\figurename~\ref{fig:OverviewADNet} shows the overview of ADNet. The network first predicts sets of affective labels $\tilde{y} = (\tilde{\mathcal{V}}, \tilde{\mathcal{A}})$ on the dataset $\mathbb{D}$ upon physiological features learned by DPAN on the subset $\mathbb{D}_s$. While DPAN has shown its superiority in classifying emotions, its predicted labels has contained noise. To overcome this issue, ADNet comprises a sub-network, which learns to map noisy labels $\tilde{y}$ to clean labels $y$, conditioned on physiological features from ConvLSTMs in DPAN. The network is supervised by the human reported labels $y$ and follows a residual architecture so that it only needs to learn the difference between the noisy and clean labels. In particular, to handle the sparsity in noisy labels, ADNet encodes the emotional occurrence $\tilde{y}$ of each of the $d$ classes in valence $\tilde{\mathcal{V}}$ and arousal $\tilde{\mathcal{A}}$ ratings into a pair of $d$-dimensional vector $[0,1]^d$. Similarly, the model projects the physiological features $\mathcal{X}$ in to a low dimensional embedding, then all embedding vectors from the two modalilties are concatenated and transformed with a hidden linear layer followed by a projection back into the high dimensional label space. At the same time, the network shares the physiological features extracted by ConvLSTMs in DPAN and learns to directly predict clean labels $\hat{y}$ following a sigmoidal function. The predicted labels $\hat{y}$ are supervised by either the output of DPAN or a human. 

To train the ADNet we formulate a joint loss function as below: 
\begin{small}
\begin{equation}
  \mathcal{L}_c = \sum_i|\bar{y}_i - y_i|  - \sum_j[u_j \log(\hat{y}_j) + (1-u_j)\log(1-\hat{y}_j)],
\label{eq:ADNetLossFunction}
\end{equation}
\end{small}where $u_j$ is $y_j$ when the SAM-ratings are available, otherwise $\bar{y}_j$. The LC-loss $\mathcal{L}_c$ is the combination of 1) the difference between the cleaned labels and the corresponding ground truth verified labels and 2) the cross-entropy to capture the difference between the predicted labels and the target labels. The cross-entropy term is only propagated to $\hat{y}_j$. The cleaned labels $\bar{y}$ are treated as constants with respect to the classification and only incur gradients from the LC-loss.

\subsection{Affective Causality Network}  
ACNet aims to solve the problem in (\ref{eq:ProblemStatementCausality}). The goal is to derive the affective causality by analyzing emotional changes and the user's relevant situations. Suppose the emotion sequence $\mathcal{M}(1:T)$ is a binary stochastic process during a discrete time interval [$1, T$], where $T$ is the maximal length of the interval. Then, an element $\mathcal{M}_j$ in the sequence $\mathcal{M}(1:T)$ is a binary indicator of the occurrence of a certain emotion $j$ at time $t$. With this notation, $\mathcal{M}^{\eta}_j$ is generated with each concatenated element $\mathcal{M}^{\eta}_j(t) = (\mathcal{M}_j(t-\eta), \ldots, \mathcal{M}_j(t))$ at each timestamp $t \leq T$. In the same way, an element $\mathcal{C}_i$ in a situation sequence $\mathcal{C}(1:T)$ is a binary indicator of the occurrence of certain situation $i$ at time $t$. Based on these notations, the time-varying affective network and the affective causality are defined with the formulation of the related learning task. 
\begin{definition}
    The time-varying affective network is denoted as $\mathcal{G} = (\mathcal{M}, \mathcal{C}, E)$, where $\mathcal{C}_i$ and $\mathcal{M}_j$ are the situation $i$ and the emotional state $j$ sequences, and $e_{ij} \in E$ indicates the affective causal effect of the sequence $\mathcal{C}_i$ on the sequence $\mathcal{M}_j$. 
    \label{def:AffectiveNetwork}
\end{definition}
\begin{definition}
    Given $n$ situation and emotion sequences $\{\mathcal{C}_1, \ldots, \mathcal{C}_i, \mathcal{M}_1, \ldots, \mathcal{M}_j \}$ from a user, the directed graph $\mathcal{G}$ is constructed in the underlying affective causality model while addressing confounding factors. 
    \label{def:AffectiveCausality}
\end{definition}

By representing each of the emotion and situation variables as a node, ACNet could be formulated as a directed graph $\mathcal{G}$ over the variable $\mathcal{C}_i \cup \mathcal{M}_j$ such that each edge $\mathcal{C}_i \rightarrow \mathcal{M}_j$ indicates the affective causal effect of sequence $\mathcal{C}_i$ on sequence $\mathcal{M}_j$. In the network, there is a parameter associated with a node or an edge between $\mathcal{C}_i$ and $\mathcal{M}_j$. Based on the descriptions above, a network generally consists of two components: 1) a directed graph $\mathcal{G}$ used to model the causal structure among the sequences and 2) association parameters on the edges used to model the causal phenomena. These two components characterize the global and local properties of the interaction between emotions and life contents, respectively, the combination of which provides a general guideline ruling the events happening on the two sequences. 

Affective causality is a computational asymmetric measure to determine the causal relationship between affective dynamics and situations. The computation is based on conditional independence testing to detect the relationship with latent confounders underneath the observational two sequences. The causal structure learning method is inspired by Cai \textit{et al.}'s work~\cite{cai2017understanding}, but differs in that the ACNet has two different variables $\mathcal{M}$ and $\mathcal{C}$ as input nodes, considering sparse relations. The proposed model analyzes this causal problem, considering 1) with and 2) without confounding factors.

\subsubsection{Learning without confounding factors}
The interaction without any latent factors can observe that the state $\mathcal{C}_i$ at timestamp $t-1$ affects the state of $\mathcal{M}_j$ at timestamp $t$, while the reverse does not hold. This observation can be formalized in the language of statistical testing with Lemma \ref{lemma1}. 
\begin{lemma}
\label{lemma1}
Given two dependent sequences $\mathcal{C}_i \rightarrow \mathcal{M}_j$ without a connected latent variable, the following asymmetric dependence relations hold: 1) there exists a delay $\eta_c$ satisfying $\mathcal{C}_i \perp \mathcal{M}^1_j | \mathcal{C}^{\eta_c}_i$ and there does not exist a delay $\eta_m$ satisfying $\mathcal{M}_j \perp \mathcal{C}^1_{i} | \mathcal{M}^{\eta_m}_j$.
\end{lemma}
\begin{proof}
Suppose $\eta_m$ and $\eta_c$ are the influence lags of $\mathcal{M}$ and $\mathcal{C}$, respectively. In the case of the 1), the state of $\mathcal{C}_i$ at $t$ time is determined only by its previous states $\mathcal{C}^{\eta_c}_i$. Hence, $\mathcal{C}_i \perp \mathcal{M}^1_j | \mathcal{C}^{\eta_c}_i$ naturally holds. In the case of the 2), there is no variable $\eta_m$ to $\mathcal{M}_j \perp \mathcal{C}^1_i | \mathcal{M}^{\eta_m}_j$ because the state of $\mathcal{M}_j$ at $t$ time is directly influenced the previous state of $\mathcal{C}_i$ at $t-1$.
\end{proof}

\subsubsection{Learning with confounding factors}
Sequences under the existence of confounding factors 1) behave similarly because of common characteristics or 2) interact independently over the timeline. Both the situation $\mathcal{C}_i$ and emotion $\mathcal{M}_j$ sequences are affected by a constant latent variable. In such a case, $\mathcal{C}_i(t)$ is thus dependent on the $\mathcal{M}_j(t-1)$ given any previous states of $\mathcal{C}_i$. Similarly, $\mathcal{M}_j(t)$ also depends on $\mathcal{C}_i(t-1)$ given any previous states of $\mathcal{M}_j$. The confounding factor, which is itself an independent variable over time, is underneath the two sequences under test. When this case occurs, a positive statistical dependence between $\mathcal{C}$ and $\mathcal{M}$ is observed. But the states of $\mathcal{C}_i(t)$ could be completely independent of $\mathcal{M}_j(t-1)$ given its previous states $\mathcal{C}_i(t-\eta_c-1:t-1)$, and $\mathcal{M}_j(t)$ could also be independent of $\mathcal{C}_i(t-1)$ given its previous states $\mathcal{M}_j(t-\eta_m-1:t-1)$. The following lemma could be recognized using the same group of conditional independence tests.
\begin{lemma}
If there is a latent factor between $\mathcal{C}_{i}$ and $\mathcal{M}_j$, the following symmetric relations hold: 1) there does not exist any delay $\eta$ satisfying $\mathcal{C}_i \perp \mathcal{M}^1_j | \mathcal{C}^{\eta_c}_i$ nor $\mathcal{M}_j \perp \mathcal{C}^1_i | \mathcal{M}^{\eta_m}_j$, 2) there exists delay $\eta_m$ and $\eta_c$ satisfying $\mathcal{C}_i \perp \mathcal{M}^1_j | \mathcal{C}^{\eta_c}_i$ and $\mathcal{M}_j \perp \mathcal{C}^1_{i} | \mathcal{M}^{\eta_m}_j$, respectively.
\end{lemma}

\begin{proof}
For 1) they are dependent on each other in the given condition set without the latent factor because $\mathcal{C}_i$ and $\mathcal{M}_j$ are dependent on the latent factor. For 2), suppose $\eta_c$ and $\eta_m$ be the self-influence lag of $\mathcal{C}_i$ and $\mathcal{M}_j$, respectively, the latent factor at time $t$ is independent of the latent factor at time $t-1$. Therefore, $\mathcal{C}_{i} \perp \mathcal{M}^1_j | \mathcal{C}^{\eta_c}_i$ and $\mathcal{M}_j \perp \mathcal{C}^1_i | \mathcal{M}^{\eta_m}_j$ hold. 
\end{proof}
Based on the above lemmas, the following theorem confirms the correctness of asymmetric measure used for causal direction detectioni with the following statements.

$S_1$: $\exists\eta_c$ satisfying $\mathcal{C}_i \perp \mathcal{M}^1_j | \mathcal{C}^{\eta_c}_i$

$S_2$: $\exists\eta_m$ satisfying $\mathcal{M}_j \perp \mathcal{C}^1_i | \mathcal{M}^{\eta_m}_j$

\begin{theorem}
Given two sequences $\mathcal{C}_i$ and $\mathcal{M}_j$, the following propositions on the causal structure between the two sequences hold: 1) $\mathcal{C}_i \rightarrow \mathcal{M}_j$ in $\mathcal{G}$, if $S_1 \wedge \neg S_2$, 2) 1) $\mathcal{C}_i \leftarrow \mathcal{M}_j$ in $\mathcal{G}$, if $\neg S_1 \wedge S_2$, and 3) there is a latent factor between $\mathcal{C}_i$ and $\mathcal{M}_j$, if $S_1 \wedge S_2$ or $\neg S_1 \wedge \neg S_2$.
\end{theorem}
\begin{proof}
    Assuming there are only three directions between two sequences $\mathcal{C}_i$ and $\mathcal{M}_j$; 1) $\mathcal{C}_i \rightarrow \mathcal{M}_j$, 2) $\mathcal{M}_j \rightarrow \mathcal{C}_i$, and 3) $\mathcal{C}_i \leftarrow \mathcal{H} \rightarrow \mathcal{M}_j$. $\mathcal{H}$ is a latent factor. For the first case, $S_1 \wedge \neg S_2 \Rightarrow \mathcal{C}_i \rightarrow \mathcal{M}_j$, and the causal structure of $S_1 \wedge \neg S_2$ based on Lemma~\ref{lemma1} cannot be $\mathcal{C}_i \leftarrow \mathcal{M}_j$. There is no confounding factor between the two sequences. Therefore, the causal structure must be $\mathcal{C}_i \rightarrow \mathcal{M}_j$. 
\end{proof}

\subsubsection{Affective Causality Direction Learning Algorithm}
\begin{algorithm}[b]
\caption{Affective Causality Direction Learning}
\begin{algorithmic}[1]
\renewcommand{\algorithmicrequire}{\textbf{Input:}}
\renewcommand{\algorithmicensure}{\textbf{Output:}}
\REQUIRE $\mathcal{C}_i$ : situation $i$ sequence
\\ $\mathcal{M}_j$: emotion state $j$ sequence 
\\ $\alpha$ : confidence threshold
\\ $T$ : maximal timestamp
\\ $\mathcal{F}$: affective pair set $\{\mathcal{C}_i, \mathcal{M}_j\}$ 
\ENSURE  The directed graph $\mathcal{G}$ 
\\ \textit{Initialization} : set $\mathcal{G}$ as empty set;
\FOR {each $i$ and $j$ in a pair $\mathcal{F}$ in a situation}
	\IF {$\nexists \mathcal{F}' \in \mathcal{F} - \{\mathcal{C}_{i}, \mathcal{M}_{j} \}, \mathcal{C}_{i} \perp \mathcal{M}_{j}|\mathcal{F}'$}
  		\STATE Test $S_1$ on $\mathcal{C}_{i}$ and $\mathcal{M}_{j}$;
		\STATE Test $S_2$ on $\mathcal{M}_{j}$ and $\mathcal{C}_{i}$;
		\IF {$S_1 \wedge \neg S_2$}
			\STATE Add $i \rightarrow j$ into $\mathcal{G}$;
		\ELSIF { $\neg S_1 \wedge S_2$ }
			\STATE Add $j \rightarrow i$ into $\mathcal{G}$;
		\ELSE 
			\STATE Add $i \leftarrow \mathcal{H} \rightarrow j$ into $\mathcal{G}$;
		\ENDIF
  	\ENDIF
\ENDFOR
\RETURN $\mathcal{G}$
\end{algorithmic}
\label{alg:AffectiveCausalityDirectionLearning}
\end{algorithm}

Algorithm \ref{alg:AffectiveCausalityDirectionLearning} describes the asymmetric relations on all dependent sequence pairs of situation $\mathcal{C}$ and emotionial state $\mathcal{M}$ and detects the directions of causal edges on the underlying the affective model $\mathcal{G}$. In this algorithm, the affective sequence set $\mathcal{F} = \{ \mathcal{C}, \mathcal{M} \}$ along with the maximal timestamp $T$, the number of sequence $N$, and the confidence threshold $\alpha$ are used as inputs for the test. Given the inputs, each affective pair is tested to see whether each is independent of the other conditioned on other variables. The complexity of Algorithm \ref{alg:AffectiveCausalityDirectionLearning} is determined by the dependent pair detection and the test of $S_1$ and $S_2$. 
\section{Evaluation on Synthetic Dataset}
In this section, we examine the robustness of the proposed ALIS on the two datasets: a public dataset and a synthetic dataset for quantitative evaluation of ADNet in emotion recognition and ACNet in causality identification.  
\subsection{Affective Dynamics Network Performance}
For quantitative evaluation of ADNet in emotion classification, we performed realistic experiments by deliberately manipulating labels on a public dataset called DEAP~\cite{koelstra2012deap}.
\subsubsection{DEAP dataset}
DEAP is a public dataset of physiological signals to analyze emotions quantitatively on a 2D plane along with valence and arousal as the horizontal and vertical axes. Its physiological signals are recorded from 32-channel EEGs at a sampling rate of 512Hz using active AgCl electrodes placed according to the international 10-20 system and 13 other peripheral physiological signals from 32 participants while they watched 40 one-minute-long excerpts of music videos. The dataset rates emotions with respect to continuous valence, arousal, liking, and dominance scaled from 1 to 9 and discrete familiarity on scales from 1 to 5 using Self-Assessment Manikins~\cite{koelstra2012deap}. 

\subsubsection{Dataset Configuration}
Since ACC gathers EEG signals from a pair of electrodes and a PPG signal, we retrieved data from the eight selected pairs of electrodes and a plethysmograph on the DEAP dataset. The 64 combinations of physiological signals per video generates 81,920 physiological data points. We split the dataset into fifths: one-fifth for testing and the remaining for training and validation. The data were high-pass filtered with a 2 Hz cutoff frequency using EEGlab and the same blind source separation technique as in~\cite{kim2018deep} for removing eye artifacts in EEG signals. A constrained independent component analysis (cICA) algorithm was applied to remove motion artifacts in PPG signals. The datasets for training, testing, and validation were subject-independent, which means they were chosen entirely randomly while keeping the distribution of ratings balanced. We used the balanced datasets to evaluate the performance of our models and other methods on solving the limited amount of clean-label problems. To evalutate the unbalanced labeling problem, we stochastically changed the amount of physiological signals associated with a label while varying the distribution of labels in the training dataset. The unbalanced dataset comprised $p$ percentage of physiological signals according to a pair of random labels in valence and arousal, with the others distributed equally. The test data remained unperturbed to allow us to validate and compare our model to other methods. The highlighted one-minute EEG and plethysmography signals were split into 6 frames of 10 seconds each. They were down-sampled to 256 Hz and their power spectral features were extracted.

\subsubsection{Evaluated Methods and Metrics}
We evaluated the performance of our ADNet, comparing the results with state-of-art methods that have shown their performance on the DEAP dataset in terms of loss functions and the number of layers. Our proposed system and the following comparative models consist of 256 hidden states and 5$\times$5 kernel sizes for the input-to-state and state-to-state transition. They were trained by learning batches of 32 sequences and back-propagation through time for ten timesteps. The momentum and weight decay were set to 0.7 and 0.0005, respectively. The learning rate starts at 0.01 and is divided by 10 after every 20,000 iterations. We also performed early-stopping on the validation set.  
\begin{itemize}
    \item Model A: 1-layered ConvLSTMs with a softmax layer as a baseline classifier. The results from the model represents a typical performance of deep neural networks (DNNs). 
    \item Model B: 1-layered ConvLSTMs with the temporal margin-based loss function as in \cite{kim2018deep}. The model has shown its superiority in recognizing human emotions, increasing the distinctiveness of physiological characteristics between correct and incorrect labels.  
    \item Model C: The model is 4-layered ConvLSTMs with a softmax layer. The model was implemented as in \cite{tripathi2017using}. 
\end{itemize}
All models learn physiological signals to classify $d^2(d=2,3,4)$ affective states, which are subdivided into each of $d$ discrete ratings for valence and arousal on 2-D emotional space. We choose the mean average precision (MAP) as a metric to evaluate the performance of our system with additional \textit{Precision}, \textit{Sensitivity}, and \textit{Specificity} to report how commonly occurring classes in a training set affect the model performance.  

\subsubsection{Evaluation Results}
\begin{figure*}[t]
    \centering
    \includegraphics[width=2\columnwidth]{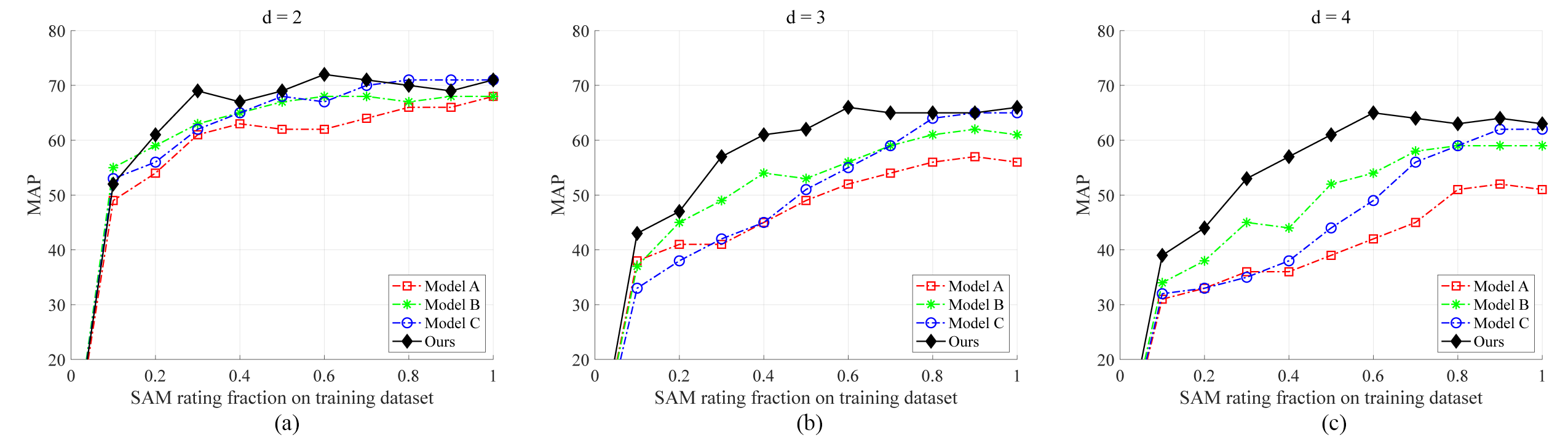}
    \caption{Test classification performance on the DEAP dataset with varying the fractions of SAM rated labels over different number of classes (valence, arousal).}
    \label{fig:MAPonDEAP}
\end{figure*}
\begin{figure*}[t]
    \centering
    \includegraphics[width=2\columnwidth]{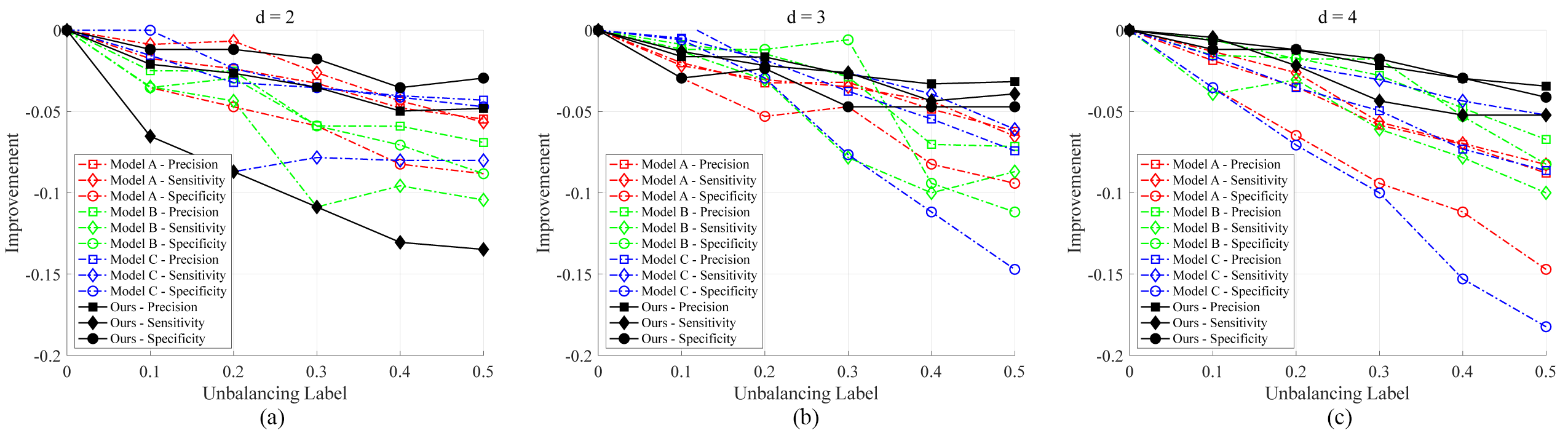}
    \caption{Test classification performance on the DEAP dataset with varying label frequencies over different classes.}
    \label{fig:ADNet_Performance_Unbalanced}
\end{figure*}
\figurename~\ref{fig:MAPonDEAP} shows that our proposed system performed better than the alternatives on the balanced dataset over all dimensions ($d=2,3,4$) unless the amount of clean labels was very large ($>0.7$). In general cases, the model C achieved the next best classification results. Its performance increased rapidly as the number of SAM-rated labels increased. On the other hand, model A consistently reported the lowest precision consistently over all configurations. This consistency could imply that model A was overfit and overconfident to some labels. With respect to overfitting, although having a specialized loss function (model B) and increasing the number of layers in DNNs (model C) improved discriminative power to classify different emotions, the limited number of clean labels seems to lead to overfitting for many classes. The three comparative methods showed worse precision when the number of labels $d$ increased, alleviating the problem of overfitting. Our approach, on the other hand, does not face this problem. In particular, it performed significantly better than the other methods when limited labels are available from $0.2$ to $0.4$. \figurename~\ref{fig:ADNet_Performance_Unbalanced} provides a closer look at how different label frequencies affect the performance in emotion recognition when the label balances were collapsed by $p$ in terms of precision, sensitivity, and specificity. The results show that the unbalancing in labels increased by the number of labels generally deteriorated the classification performance. In particular, the vertically stacked layers in model C led it to misunderstand the physiological characteristics associated with a major numbers of labels. While the other three models became overfit to common labels and reduced the overall accuracies, our approach is most effective for very common classes and shows improvement for all but a small subset of rare classes.      
\subsection{Affective Causality Network Performance}
\subsubsection{Experimental Setup}
To evaluate the affective causality identification on the proposed ACNet, a synthetic data set was generated. In the data set, pairs of affective sequences $\mathcal{F}$ are generated by simulating a Poisson point process with 10 min per timestamp and $\epsilon$ as occurrence frequency for situations $\mathcal{C}$ and emotions $\mathcal{M}$ per day. Each sequence is influenced by its causal nodes as the time-dependence influence function with exponential probability $p(\bigtriangleup_t, \eta) = \eta e^{(-\bigtriangleup_t \eta)}$, where $\bigtriangleup_t$ is the time interval between $t$ and the causal sequence's most recent state, and $\eta$ is the average influence lag. 

\subsubsection{Evaluated Methods and Metrics}
We evaluated the performance of ACNet, comparing it against the transfer entropy (TR) method~\cite{ver2012information} and Granger's causality (GC) method~\cite{seth2012assessing}. Two groups of causal structures were designed for experiments. The first group has the causal structures of directed graphs without latent variables, and the second group consists of directed graphs with latent factors. Given $n$ sequences and the average in-degree $d_g$ of the graph, a pair of two emotion and situation nodes and a directed edge between the pairs into the graph are selected until there are $d_g \cdot n$ edges in the graph. For the second group, confounding factors are selected by $n_c$ independent pair of nodes and an additional latent factor $\mathcal{H}$ underlying two edges is added into the causal graph. We choose \textit{Precision}, \textit{Recall}, and \textit{F1}-score as metrics to evaluate each type of causality.  

\subsubsection{Results}
\begin{figure*}[t]
    \centering
    \includegraphics[width=2\columnwidth]{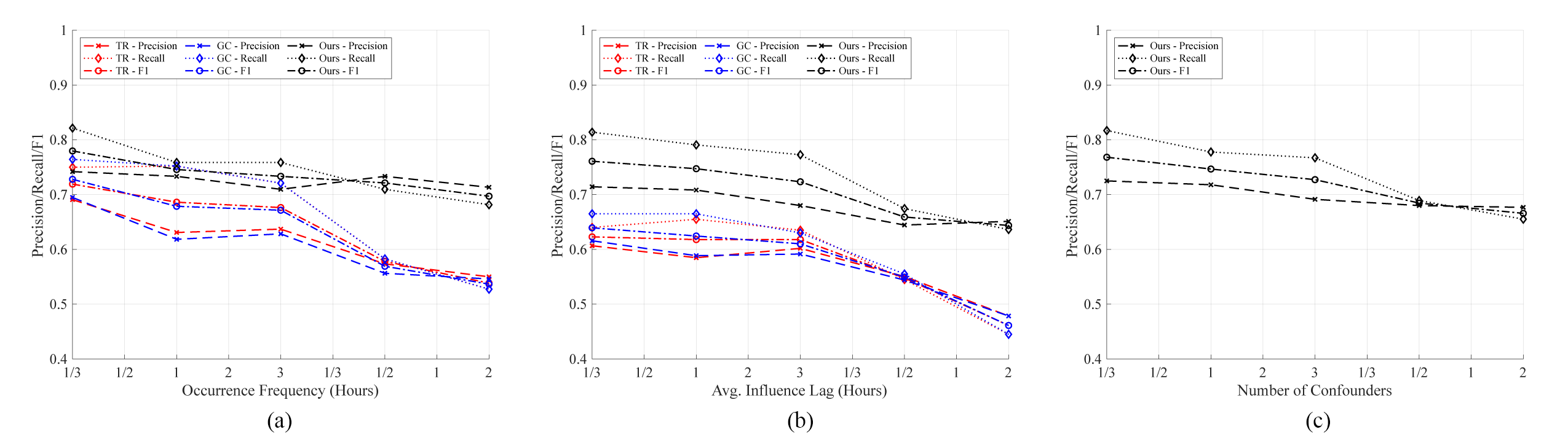}
    \caption{Affective identification performance on the synthetic dataset varying occurrence frequency, average influence lag, and numbers of confounder parameters}
    \label{fig:ACNet_Performance}
\end{figure*}
\figurename~\ref{fig:ACNet_Performance} shows the performance of ACNet and the two other methods on different causal structures and data generation parameters with varying occurrence frequencies $\epsilon$, and average influence lag $\eta$. Overall, ACNet consistently outperformed the other methods. As shown in \figurename~\ref{fig:ACNet_Performance}(a), the occurrence frequency reflects the sparsity of the sequences such as the number of timestamps with recorded situations. The performance of TR and GC declines with increasing sparsity. On the other hand, ACNet maintained its performance at around 0.7 even when there was only one occurrence in three hours. This result implies that our model is effective in solving the causal discovery problems on the sparse affective situation sequences. \figurename~\ref{fig:ACNet_Performance}(b) shows the sensitivity with different influence lags. The recall of our model decreased with the increasing average influence lags because larger influence lags means longer dependence on the previous states. The precision of our model was not sensitive to the influence lag, keeping a performance between 0.65 and 0.8, while the precision of other methods decreased as the influence lag increased. These results imply our model is capable of catching long-term dependence for learning causal direction.  

\section{Evalutation on Real-World Dataset}
Underlying the quantitative analysis of our model on the public and synthetic datasets, we first built a real-world dataset called the Affective Lifelog Dataset, where participants used our ACC in their daily life. We then evaluated the performance of the proposed ALIS with respect to physiological discrimination in emotion recognition and user agreement in causal identification. 

\subsection{Affective Lifelog Dataset}
\begin{figure}[t]
    \centering
    \includegraphics[width=1\columnwidth]{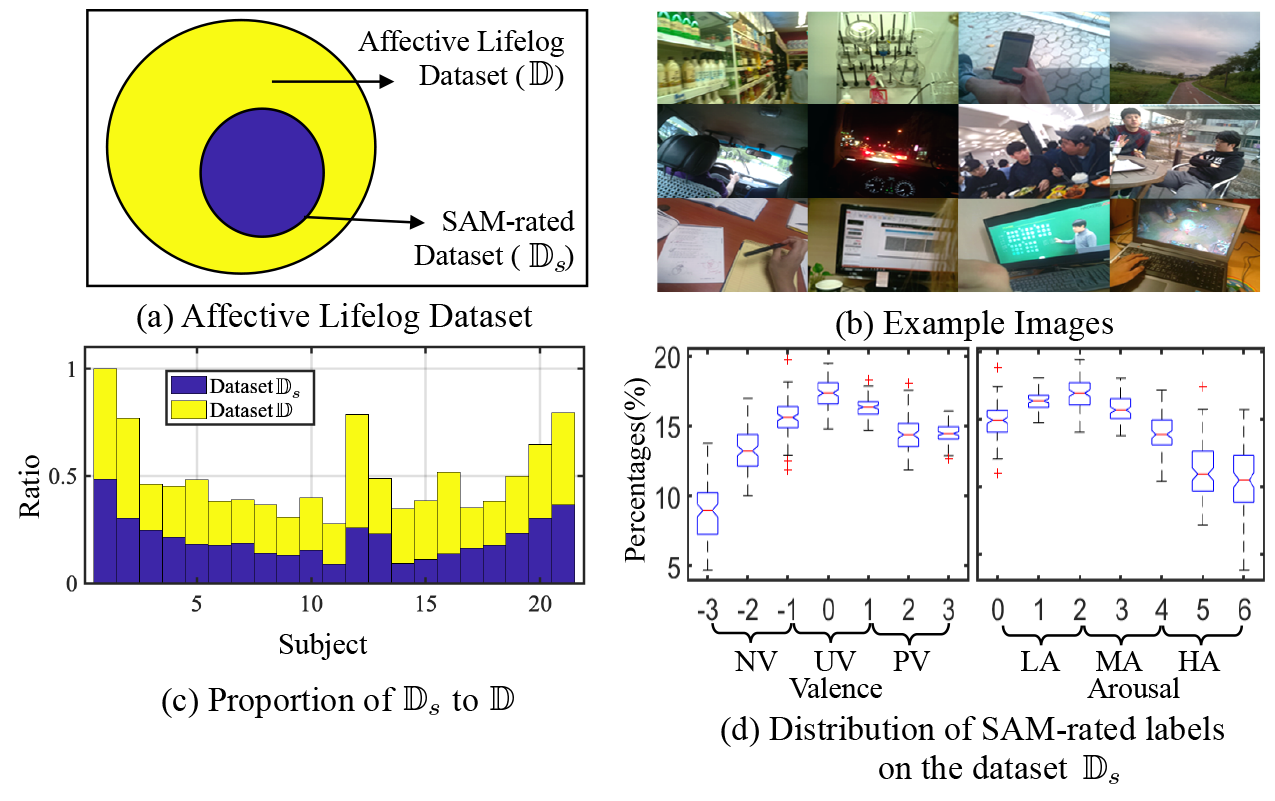}
    \caption{Affective Situation Dataset $\mathbb{D}$ and its subset $\mathbb{D}_s$ which contains SAM-rated situations. (b) Example images in the dataset $\mathbb{D}$. (c) Proportion of the subset $\mathbb{D}_s$ and the dataset $\mathbb{D}$. $N_s = 378$ for the subject 1.(d) Distribution of the SAM-rated situations in valence and arousal labels on the subset $\mathbb{D}_s$.}
    \label{fig:DatasetOverview}
\end{figure}
\subsubsection{Data Acquisition}
The dataset consists of two modalities of physiological signals, accelerometer signals, and frontal images obtained by ACC from 16 male and 5 female university students aged from 21 to 35 (26.4 $\pm$ 4.87) years. Our requirement for participation was to perform at least one common task of a university student, such as conducting research, taking classes, or having a discussion with colleagues. We required them to wear the device over six hours per day for up to 45 days with \$10 compensation per day. This experiment was approved by the Institutional Review Board (IRB) in Human Subjects Research. 

To identify individual specific affective contents and assign the SAM ratings to them as ground-truth labels, we asked the participants if they had any affective contents that had elicited a specific feeling, and how the contents changed their emotion before and after the elicitation. The life contents were perceived as breakthroughs to change their mentality. The changes were rated by the SAM scaled from 0 to 6 for arousal and valence ratings. Furthermore, we retrieved the following affective contents manually which are considered to potentially affect mental status stress: 1) watching movies, 2) drinking coffee, 3) drinking green tea, 4) hanging out with friends, 5) playing games, 6) drawing a picture, 7) taking a walk, 8) reading a book, 9) eating food, 10) studying at a desk, 11) reading research papers on a computer monitor, 12) conducting researches in a laboratory, and 13) playing with media devices. In such situations, the participants performed the SAM ratings every five days. 

\subsubsection{Affective Lifelog Dataset}
\figurename~\ref{fig:DatasetOverview} summarizes the Affective Lifelog Dataset $\mathbb{D}$. The situations rated by the SAM consist of a subset $\mathbb{D}_s$ with pairs of emotion labeling $y$ = ($\mathcal{V}$, $\mathcal{A}$) scaled from 0 to 6 for valence $\mathcal{V}$ and arousal $\mathcal{A}$ ratings. The labeling $y$ = ($\mathcal{V}$, $\mathcal{A}$) is used as ground-truth to evaluate the performance of our proposed system. As shown in \figurename~\ref{fig:DatasetOverview}(c), the dataset has a limited amount of labeling data. The ratios of the dataset $\mathbb{D}$ and the subset $\mathbb{D}_s$ are 0.418 from all participants. Furthermore, the distribution of labels is unbalanced. These challenge issues are consistent with our understanding describe in (C.1) and (C.2). \figurename~\ref{fig:DatasetOverview}(b) shows some situations in the dataset $\mathbb{D}$ from our real-world experiments. Participants have experienced various situations in daily life such as driving a car, reading a research paper, playing games, and taking a walk. 

\subsection{Experimental Setup}
\subsubsection{Train/Test/Evaluation on the Dataset $\mathbb{D}$}
From the real-world dataset $\mathbb{D}$, we grouped affective labels ($\mathcal{V}$, $\mathcal{A}$) in valence and arousal into seven discrete affective states: NVHA, NVMA, NVLA, UVLA, PVLA, PVMA, and PVHA, comprising low (LA), mid (MA), and high (HA) arousal and negative (NV), neutral (UV), and positive (PV) valence ratings. We should note that two states (UVMA and UVHA) were omitted since their occurrence was extremely low. 40\% of the dataset $\mathbb{D}_s$ was used for training. We then retrieved 10,000 pieces of physiological data per affective state, which is 70,000 pieces of physiological data on total of seven states for every participant on the dataset $\mathbb{D}$. ADNet first predicts $\tilde{y}$ on the dataset $\mathbb{D}_s$ and produces pseudo-labels $\bar{y}$ with learning physiological patterns $\mathcal{X}$. Then, the model uses the patterns with the SAM-rated rating $y$ to produce the clean labels $\hat{y}$. The test data remained unperturbed to validate and compare our model to other methods. Since physiological signals, in particular, EEG signals, are vulnerable to motion artifacts~\cite{daly2013automated}, we developed a strategy to improve the quality of EEG signals by abandoning EEG signals highly correlated with motion artifacts rather than separating and removing motion artifacts in EEG signals occurring due to body movement~\cite{daly2012does, alarcao2017emotions}. To pursue this strategy, we subdivided the EEG signals into two groups separated by varying the accelerometer data. From each of the two groups we extracted the following EEG features: 1) Mean power, 2) Maximum amplitude, 3) Standard deviation of the amplitude, 4) Kurtosis of the amplitude, and 5) Skewness of the amplitude. The features are metrics to describe key characteristics of clean EEGs~\cite{jenke2014feature}. After representing the features into two-dimensional space using principal component analysis (PCA), we computed the Bhattacharyya distance between the two groups over the two-dimensional space as a differentiator between the clean EEG and the contaminated EEG based on the maximum distance between the two groups.

\subsubsection{Network Settings}
ADNet is composed of two-layered networks with 512 and 256 hidden states and has 5$\times$5 size of kernels for the input-to-state and state-to-state transitions. To train the network, we used learning batches of 32 sequences, set the learning rate as 0.01 initially, and divided the rate by ten after every 20,000 iterations. The weight decay and momentum were set to 0.0005 and 0.6, respectively. Back-propagation through time was performed for ten timesteps. We also performed early-stopping on the validation set. For ACNet, we only selected pairs of the affective sequence $\mathcal{F}$ which were associated with clean physiological signals from all participants with every 10 min as a timestamp.    
\subsection{Evaluation Results}
\subsubsection{Performance of ALIS on Recognizing Emotions and Identifying Affective Causalities}
\begin{table}[t]
\begin{scriptsize}
\renewcommand{\arraystretch}{1.3}
\caption{Performance of ADNet for classifying the seven emotional states and ACNet for identifying the affective causalities on the dataset $\mathbb{D}$. }
    \label{table:ALISresultLifelog}
\centering
\begin{tabular}{c||ccccccc} \hline
    ACNet & \multicolumn{7}{c}{ADNet}                      \\
      & NVHA & NVMA & NVLA & UVLA & PVLA & PVMA & PVHA \\
\hline	0.74  & 0.61   & 0.59   & 0.55   & 0.64   & 0.52   & 0.56   & 0.55  \\ 
\hline
\end{tabular}
\end{scriptsize}
\end{table}
\begin{figure}[t]
    \centering
    \includegraphics[width=1\columnwidth]{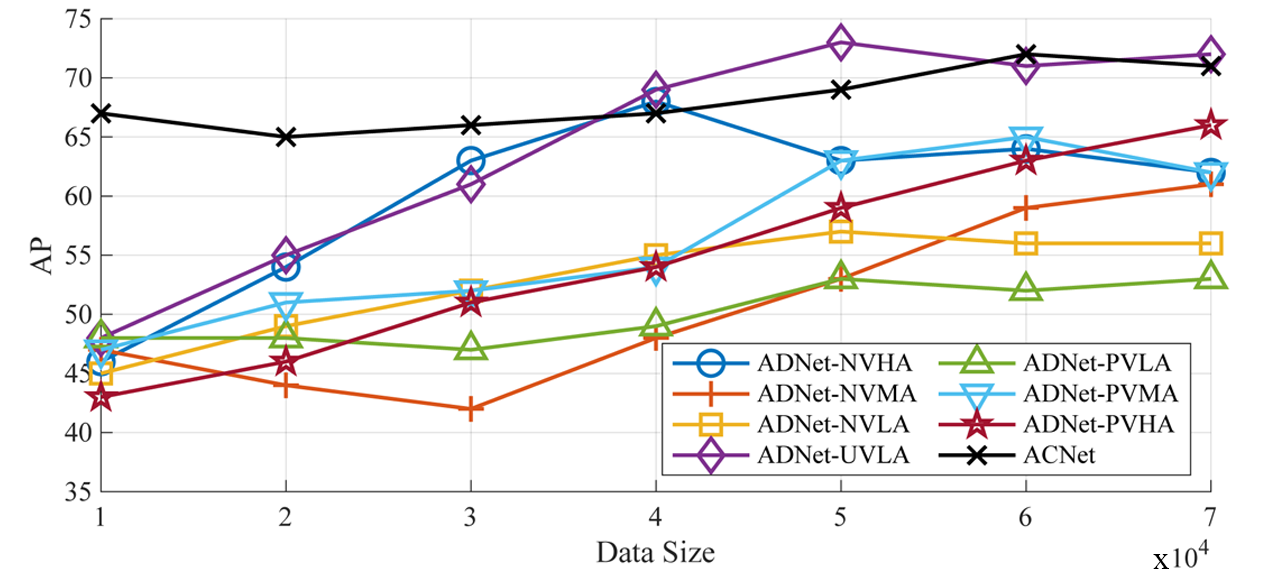}
    \caption{Performance in emotion recognition and affective causality identification on the real-world dataset $\mathbb{D}$ increasing the dataset size over different emotional states and affective situations}
    \label{fig:Accuracy_ALIS_DifferentSize_Lifelog}
\end{figure}
\tablename~\ref{table:ALISresultLifelog} reports the average precision in recognizing emotions and identifying affective causalities over all participants on the dataset $\mathbb{D}$. The performance on the UVLA state showed the highest results over all participants.

This implies that when participants are in a UVLA state, such as calm and relaxed feelings, their physiological signals fluctuate in common patterns, which helps ALIS to learn their characteristics. It can be also attributed that the percentages of affective situations rated with the labels UV and LA were higher than others on the dataset $\mathbb{D}$. On the other hand, the performance in classifying valence ratings associated with low arousal ratings besides LAUV, such as NVLA and PVLA states, were relatively lower than other arousal ratings. In particular, PVLA had the lowest accuracy. This observation may imply that classifying emotions by valence can be improved by considering their associations with arousal. Affective causalities in the 13 situations were identified with a precision of 0.74. Although ACNet identified affective causalities by regarding the prior results of ADNet recognizing affective states, the performance was consistently higher than on detecting emotional states. This can be attributed to ACNet working well on eliminating false causal relationships built on wrong emotions. 

Unlike the DEAP dataset, our real-world dataset $\mathbb{D}$ records everyday activities with physiological signals. Hence, the inter- and intra-day variability in physiological signals can determine the performance of ALIS in understanding affective dynamics. \figurename~\ref{fig:Accuracy_ALIS_DifferentSize_Lifelog} shows the performance in emotion recognition and causality identification on different amounts of physiological data in days. The accuracy generally improved when sufficient data was available. ADNet classified HA states better when more of their associated signals were provided. On the other hand, classifying negative states such as NVLA and PVLA showed the smallest improvement with increased data sizes. These results could imply that the affective dynamics in valence requires the development of elaborate deep learning architectures more than the provision of sufficient physiological data. Since causal identification depends on the prior emotion recognition, the performance in emotion recognition is another key in causality. However, this experiment shows that our system does not face this problem. It consistently achieved high scores with small increaments. 

\subsubsection{Analysis of Emotion Recognition}
\begin{figure}[t]
    \centering
    \includegraphics[width=1\columnwidth]{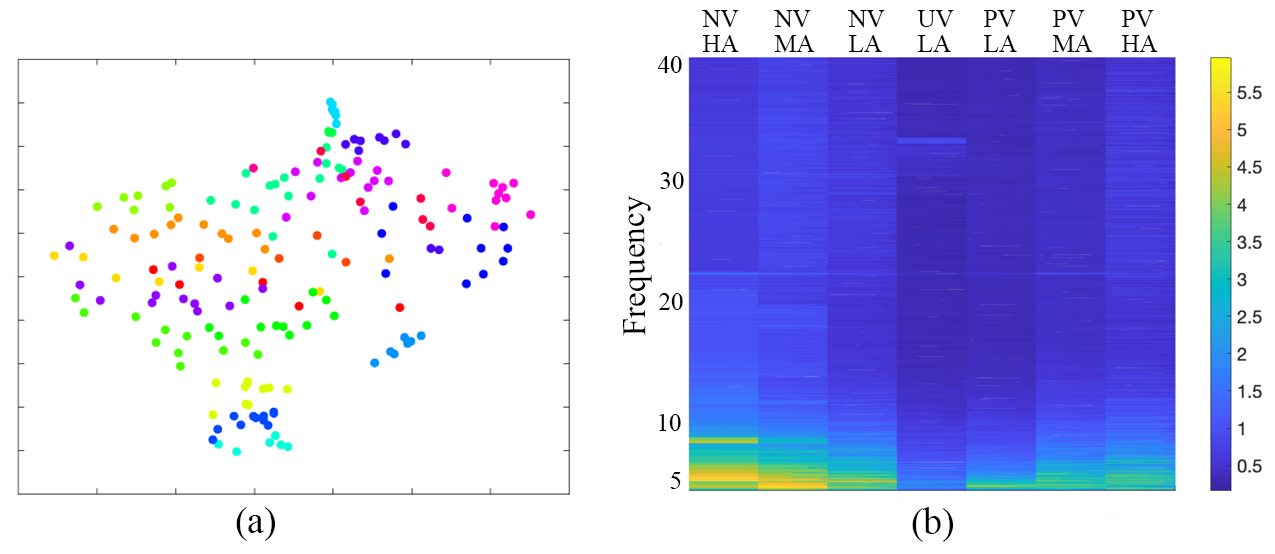}
    \caption{(a) Clustering of the UVLA states from the participants using t-SNE applied to physiological features $\mathcal{X}$. (b) The shared physiological representation visualized by the grand average of the feature $\mathcal{X}$ over frequency bands in the seven affective states. }
    \label{fig:tsne_grandmean}
\end{figure}
\figurename~\ref{fig:tsne_grandmean}(a) shows the multimodal physiological features $\mathcal{X}$ in a two-dimensional space obtained using t-distributed stochastic neighbor embedding (t-SNE)~\cite{maaten2008visualizing}, a popular technique for unsupervised dimensionality reduction and visualization. The features of different participants in an affective state cluster in the projected space, revealing their high variety. Despite their heterogeneity, ADNet is capable of capturing some shared physiological characteristics. To investigate the physiological phenomena observed when emotions are classified using ADNet, we visualize the grand average of brain lateralization features in (\ref{eq:BrainLateralizationFeature}) across participants in valence and arousal ratings. We should note that heart-related features have served as essential elements reflecting the function of ANS. However, in this section, we focus more on brain lateralization, as it has relatively large inter- and intra-subject variability. \figurename~\ref{fig:tsne_grandmean}(b) shows the commonly shared frequencies over all participants in the seven emotional states. We found that alpha and beta bands were activated when most emotional states except UVLA and PVLA were elicited. Strong negative related feelings such as NVHA and NVMA led to an increase in physiological changes in alpha and beta. The other two states were characterized by either the theta or gamma band. This indicates emotional reactions in real-world situations lead to the activation of physiological signals in alpha and beta while stability in emotion maintains physiological signals in the theta or gamma bands. Although several alternatives have been suggested in reports on the neuro-physiological correlates of affective states, this result is in line with our previous work~\cite{kim2018deep}. Our findings may also be justified by the fact that when some negative but approach-related emotions, such as ``anger" that would be lateralized to the left hemisphere, are induced, they lead to increases in alpha band activity in the left anterior and the left temporal regions in the beta bands. We also found that emotional lateralization has also affects on reacting emotions in cases of arousal correlated with valence. When increasing arousal from low to high within the same level of valence states, there are slight increments in the theta band. This observation may reflect the inter-correlations between valence and arousal, as reported in~\cite{koelstra2012deap}. 

\subsubsection{Causality Discovery}
\begin{figure}[t]
    \centering
    \includegraphics[width=1\columnwidth]{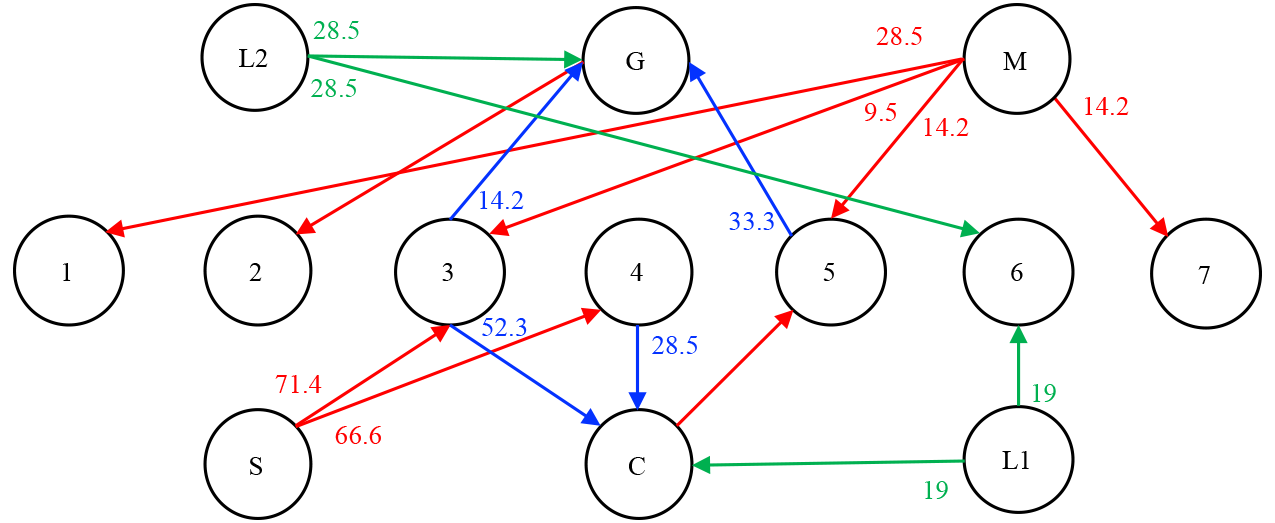}
    \caption{Causal Structures between the four situations and the seven emotional states on the real-world dataset $\mathbb{D}$. Nodes from 1 to 7 represent affective states NVHA, NVMA, NVLA, UVLA, PVLA, PVMA, and PVHA. Situation nodes are denoted as S, G, C, M and L for \textit{Studying on a desk}, \textit{Playing Games}, \textit{Drinking coffee/green tea in a cafeteria}, \textit{Watching movies}, and a latent factor. Red lines indicates the asymmetric causality from the situation to the emotion states. Blue indicates the opposite causalities. Green lines show the latent factors between the situation and emotional states. }
    \label{fig:ResultCausalIndentificationLifelog}
\end{figure}
To better understand the overall causal pairs in daily life, we report three types of causal networks of the four frequent stress relievers on the dataset $\mathbb{D}$. We choose the four most frequent situations: ``Studying at a desk,'' ``Playing games,'' ``Drinking coffee/green tea in a cafeteria,'' and ``Watching movies.'' The situations and their causality were manually browsed in the results of our system. \figurename~\ref{fig:ResultCausalIndentificationLifelog} shows the causality structure over all participants. Edges are asymmetric affective causalities between two nodes. Numbers on edges are percentages of causalities affected by the relationship. Note that the graph does not include cascade flows; namely, in case of $1 \rightarrow 2 \rightarrow 3$, it has only a causal relation between node $1 \rightarrow 2$ and $2 \rightarrow 3$, but the causality does not propagate $1 \rightarrow 3$.

Participants changes their behaviors when they feel specific emotions. First of all, we found that while studying at a desk, most users felt the emotional states NVLA and UVLA, which are close to calm and stress, but any feeling did not affect students' desire to study. This phenomenon indicates that the activity can be a major emotional cause regulating their feelings of being stressed in repeated school routines. In other words, the participants has studied either habitually or for no particular emotional reason, but they were stressed by studying. Interestingly, these types of stressed feelings led them to change their behaviors. Some people had coffee when their feelings rated as having negative valence and low arousal. Furthermore, in line with affective causality, drinking coffee had a causal effect on overcoming emotional negativity, increasing valence. While most activities causes a given emotion to a particular feeling, ``Watching a movie" affected multiple emotional states. This can be attributed to individual emotional acceptance of a movie or characteristics of different movie genres. Similarly, when users play games, they feel either NVMA or PVMA emotional states. The polarity in valence from the two emotional states implies playing games is accompanied by emotional elements of fatigue, while it helps to lead positive feeling. From a few users, we found there exist two hidden factors between emotions and situations. When users had the latent factor $L2$, they played the game while feeling excited. Similarly, the latent factor $L1$ bridged users to drink coffee/green tea with a happy feeling. Neither connection had been established without the two factors. 
\section{Conclusion}
We presents a new wearable system to detects emotional changes and find causational relationships in daily life, based on the new affective model of interaction behavior. By applying our proposed model to real-world dataset, our approach is able to find meaningful causal connections between emotions and behaviors, even when confounder variables potentially affect human emotions and behaviors. In the future, we will explore the possibilities of social interaction behavior caused by individual emotional changes. It is also interesting and even more challenging to effectively implement causality identification in the complex human behaviors and situational analysis of daily life. 
\section*{Acknowledgment}
This work was supported by Institute for Information \& communications Technology Promotion(IITP) grant funded by the Korea government(MSIT) (No.2017-0-00432, No.2017-0-01778).

\ifCLASSOPTIONcaptionsoff
  \newpage
\fi



\bibliographystyle{IEEEtran}
\bibliography{IEEEabrv,./2018_TCYB_ALIS}

\begin{IEEEbiography}[{\includegraphics[width=1in,height=1.25in,clip,keepaspectratio]{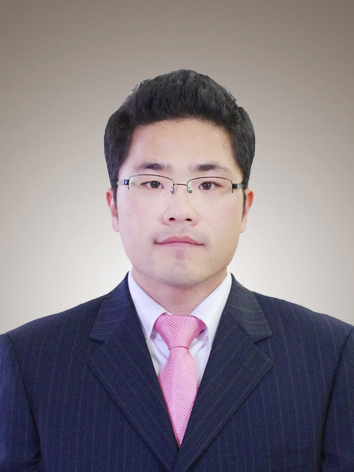}}]{Byung Hyung Kim}
received the B.S. degree in computer science from Inha University, Incheon, Korea, in 2008, and the M.S. degree in computer science from Boston University, Boston, MA, USA, in 2010. He is currently working toward the Ph.D. degree at KAIST, Daejeon, Korea. His research interests include affective computing, brain-computer interface, computer vision, assistive and rehabilitative technology, and cerebral asymmetry and the effects of emotion on brain structure.
\end{IEEEbiography}

\begin{IEEEbiography}[{\includegraphics[width=1in,height=1.25in,clip,keepaspectratio]{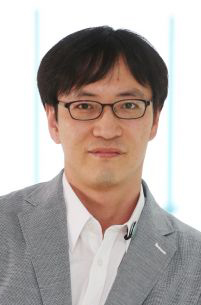}}]{Sungho Jo}
(M'09) received the B.S. degree in school of mechanical \& aerospace engineering from the Seoul National University, Seoul, Korea, in 1999, the S.M. in mechanical engineering, and Ph.D. in electrical engineering and computer science from the Massachusetts Institute of Technology (MIT), Cambridge, MA, USA, in 2001 and 2006 respectively. While pursuing the Ph.D., he was associated with the Computer Science and Artificial Intelligence Laboratory (CSAIL), Laboratory for Information Decision and Systems (LIDS), and Harvard-MIT HST NeuroEngineering Collaborative. Before joining the faculty at KAIST, he worked as a postdoctoral researcher at MIT media laboratory. Since December in 2007, he has been with the department of computer science at KAIST, where he is currently Associate Professor. His research interests include intelligent robots, neural interfacing computing, and wearable computing.
\end{IEEEbiography}

\end{document}